\documentclass[11pt,letterpaper]{article}
\usepackage{arxiv}

\usepackage{hyperref}
\usepackage{url}
\usepackage{amsmath,amssymb,amsthm,amsfonts}
\usepackage{algorithm}
\usepackage{graphicx}
\usepackage{subcaption}
\usepackage{enumitem}
\usepackage[noend]{algpseudocode}
\newtheorem{lemma}{Lemma}

\newtheorem{theorem}{Theorem}

\title{Counting Hours, Counting Losses: The Toll of Unpredictable Work Schedules on Financial Security}
\author{Pegah Nokhiz\\
  Cornell University, Cornell Tech\\
  \texttt{pegah.nokhiz@gmail.com} \\
  \And
   Aravinda Kanchana Ruwanpathirana \\
  National University of Singapore\\
\texttt{kanchana.ruwanpathirana@gmail.com}\\
  \And
  Aditya Bhaskara \\
 University of Utah\\
  \texttt{bhaskaraaditya@gmail.com} \\
  \And
  Suresh Venkatasubramanian \\
  Brown University\\
  \texttt{suresh@brown.edu} \\}

\begin{document}



\date{}
\maketitle
\begin{abstract}
Financial instability has become a significant issue in today’s society. While research typically focuses on financial aspects, there is a tendency to overlook time-related aspects of unstable work schedules. The inability to rely on consistent work schedules leads to burnout, work-family conflicts, and financial shocks that directly impact workers' income and assets. Unforeseen fluctuations in earnings pose challenges in financial planning, affecting decisions on savings and spending and ultimately undermining individuals' long-term financial stability and well-being.

This issue is particularly evident in sectors where workers experience frequently changing schedules without sufficient notice, including those in the food service and retail sectors, part-time and hourly workers, and individuals with lower incomes. These groups are already more financially vulnerable, and the unpredictable nature of their schedules exacerbates their financial fragility.

Our objective is to understand how unforeseen fluctuations in earnings exacerbate financial fragility by investigating the extent to which individuals' financial management depends on their ability to anticipate and plan for the future. To address this question, we develop a simulation framework that models how individuals optimize utility amidst financial uncertainty and the imperative to avoid financial ruin. We employ online learning techniques, specifically adapting workers' consumption policies based on evolving information about their work schedules.

With this framework, we show both theoretically and empirically how a worker's capacity to anticipate schedule changes enhances their long-term utility. Conversely, the inability to predict future events can worsen workers' instability. Moreover, our framework enables us to explore interventions to mitigate the problem of schedule uncertainty and evaluate their effectiveness.
\end{abstract}

\section{Introduction}
\label{sec:intro}
Financial insecurity has become a defining feature of modern life in America \cite{weller2018working}. A convergence of socioeconomic factors, escalating inequality, and precarious employment situations \cite{kalleberg2009precarious, kalleberg2011good} have rendered families more susceptible to financial shocks with irreversible long-term consequences.

Unstable and unpredictable work schedules which are widespread in many sectors, e.g., the food-service and retail sectors, play a pivotal role in this insecurity. In addition to low wages and limited benefits, economically disadvantaged workers often contend with unpredictable work schedules. Previous research reveals that 80\% of food industry workers have minimal to no influence on their schedules, and 69\% are obligated by the system to maintain schedules that are ``open and available'' for work at any given time \cite{loggins2020here, schneider2019s}. 8.4\% of workers and consultants/contractors aged 18-65 reported significant fluctuations in their income on a monthly basis. 51\% attributed their income instability to an irregular work schedule \cite{mccrate2018unstable}.

With a just-in-time work schedule, managing one's finances becomes increasingly difficult. Workers earn income through their employment and use it to fulfill daily necessities like food, shelter, transportation, clothes, recreation, and so on. In this context, the income earned by workers is allocated to different purposes, leading to corresponding gains in utility based on how much they save and consume optimally. Given that work schedule instability has a direct impact on individuals' employment and income, an important line of inquiry arises: What are the adverse effects of this scheduling discrepancy on these protected groups? Additionally, how can these potential adverse impacts be effectively measured?

To answer these questions, we need to study the resultant consumption and saving behavior of workers in the face of unforeseen financial shocks from unstable shifts. Thus the primary goal of this paper is to thoroughly investigate the dynamics and effects of work instability on the earned utility of workers. To do this, we need an agent-based behavioral model capable of simultaneously simulating several interconnected phenomena, i.e., the model should encompass how individual consumption responds to unexpected financial shocks and the behavior of individuals as they seek to maximize utility with different levels of information on future events.

However, most economic models of consumption and savings assume that individuals possess the ability to fully look into the future when making consumption decisions \cite{deaton1992understanding, deaton1989saving}. On the contrary, in real life, information is limited, and our ability to act on it is also limited. Thus, we need to understand how individuals' financial decision-making changes if they only have a limited window of information into how their financial well-being might change.  This is particularly relevant in settings where workers are increasingly only given limited visibility into income/asset-affecting decisions like how many hours they will work or at what rate they will earn.

In this paper \cite{nokhiz2024modeling}, we seek to model this problem using the language of online learning, i.e., an adaptive update of the workers' consumption policies (an online learning paradigm) where the policy is recalculated at each step as more information on work schedules becomes available. We investigate how far utility maximizing strategies depend on the degree of ``foresight''.

The modeling approach is advantageous for several reasons. First, offline models generate a singular, fixed policy incapable of adapting to new financial information/shocks. Secondly, the simultaneous occurrence of financial data availability and optimization better aligns with real-world behavior. Thirdly, the introduction of future lookahead allows for a nuanced exploration of its utility under different conditions and parameters (i.e., different levels of foresight). Lastly, the online framework with future information provides a robust simulation environment for studying the consequences of work schedule instability and bias in future information availability. 


It also opens avenues for empirically exploring mitigation strategies. This is in line with the current ongoing efforts in adopting various policies and regulations to make altering work schedules more equitable, e.g., the San Francisco Board of Supervisors' mandates (based on the Retail Workers Bill of Rights) that ensure more advance notice for hourly workers in retail chain stores when establishing or altering work schedules \cite{golden2015irregular}.



\paragraph{Contributions.} Overall, the main contributions of this work are:

\begin{itemize}






\item \textbf{Online Algorithm with Lookahead:} A novel online algorithm is proposed, capable of handling varying levels of lookahead. The algorithm allows individuals to utilize the deterministic information from the lookahead to modify their consumption decisions in real-time, responding to and adapting to financial shocks as they occur. This algorithm becomes a valuable tool for studying the impact of different degrees of foresight on decision-making.

\item \textbf{Formal Analysis of the Effects of Lookahead:} Workers who possess a lookahead benefit from an advantage that increases proportionally with the magnitude of their lookahead, as opposed to workers lacking any lookahead. Furthermore, it's important to note that this gap is tight under appropriate assumptions.

\item \textbf{Empirical Analysis of Future Information Provision:} Individuals (workers) are provided with future information regarding certain upcoming events. This aspect enables an empirical exploration of the extent to which future information aids in financial management, contributing to increased long-term utility.


\item \textbf{Temporal Equity in Workplace Schedules:} We explore temporal equity, particularly in the context of the implications of lack of advance notice (future lookahead) on work timetables that affect the disadvantaged subpopulations more acutely. 



\item \textbf{Mitigation Strategies:} Various mitigation strategies with reference to fair workplace laws and acts are explored to understand how the adverse effects of just-in-time work schedules, particularly in terms of workers' utility, can be alleviated. This investigation aims to identify effective strategies for improving the overall well-being of workers facing schedule instability.
  
\end{itemize}

The discussion of related work is presented in \S\ref{sec:related-work}. In \S\ref{sec:simulation-background}, we provide an overview of the required modeling background and introduce the online utility maximization framework. Subsequently, in \S \ref{sec:theory}, we present a formal argument underscoring the significance of future lookahead. Moving forward, we showcase the application of the framework to address questions related to lookahead (\S\ref{sec:expts-look} and \S\ref{sec:expt-return}) and interventions (\S\ref{sec:expts-interventions}). Finally, we conclude with a discussion of limitations and potential future research directions in \S\ref{sec:future-work}.

\section{Related Work}
\label{sec:related-work}
\paragraph{Work Schedule Instability.} 


The current focus of research primarily lies in the field of sociology, specifically examining irregular work scheduling and its various repercussions. Unstable schedules cause income volatility \cite{hannagan2015income, morduch2017and, farrell2016paychecks, reserve2016report, schneider2017income} and income volatility results in financial and life hardship \cite{bania2006income, reserve2016report, leete2010effect, mccarthy2018poverty,Lambert218}. This encompasses issues such as burnout from precarious work schedules \cite{schneider2019consequences, Hawkinson2023-jg} and work-family conflicts \cite{golden2015irregular, julia2014}, particularly affecting parents with unpredictable or just-in-time schedules. The impact also extends to areas like anxiety and child behavioral problems linked to parental work instability \cite{schneider2019s}. Additionally, the field of Human-Computer Interaction (HCI) has also strived to study similar repercussions with a participatory outlook \cite{Uhde2020,Wood2021,Lee2021}. They investigate the necessity of worker participation in deciding schedules to ensure fairness \cite{Uhde2020} and how to use elicitation methods to figure out worker preference models to assist in schedule management \cite{Lee2021}.

Further, statistical data shows a pronounced unfairness in advance notices for altering work timetables for underprivileged groups. Hourly workers, individuals with lower educational attainment, women of color, and specific service sectors are disproportionately affected \cite{schneider2019s, mccrate2018unstable} with managerial discretion \cite{wood2018powerful, lambert2008passing}. These work schedules make it difficult to plan for the future \cite{shah2015scarcity} and difficulty in planning would result in disproportionate financial poverty and hardship \cite{gennetian2015persistence}.

Along the same lines, there are some reports pointing to scheduling software and planning algorithms as a factor behind more unpredictable scheduling, particularly for low-wage workers in the service industry \cite{kantor2014working,cons_res,Kathleen2019,Zhang2022}. For example, a New York Times article pointed out how some employees with algorithmic schedules rarely learned their timetables more than three days before the start point of a workweek \cite{kantor2014working} or how pay reduction is correlated with sudden schedule changes and sales figures \cite{loggins2020here}. 

\paragraph{Simulation.} 
Simulation has been used in many social settings, including fairness in lending \cite{liu2018delayed}, resource allocation \cite{ensign2018runaway, elzayn2019fair}, college admissions \cite{hu2019disparate, kannan2019downstream}, financial analysis \cite{cristelli2011critical}, technology adoption \cite{klugl2023modelling}, studying supply chain shortages \cite{yongsatianchot2023agent}, and simulations of global crisis like the pandemic to minimize the spread of virus \cite{abe2022using}.
To study the effects of advance notice on individuals, we follow in the path of research by \cite{fairness-static}, \cite{zheng2020ai}, and \cite{nokhiz2021precarity, nokhiz2024agent}, which use simulation as a mechanism to study behaviors of agents in systems.


\paragraph{Consumption Models.} 


Consumption models fall within the broader category of discounted utility (DU) models that collectively contribute to our understanding of how individuals navigate decisions related to consumption and saving over time \cite{deaton1992understanding, deaton1989saving}. These models involve individuals engaging in discounted consumption utility maximization, where they generally prefer immediate rewards over future rewards of the same size \cite{chabris2010intertemporal}. The decision-making process involves choices regarding when to consume or save \cite{samuelson1937note}. Several widely used intertemporal consumption models include the permanent income hypothesis (PIH), the life-cycle model \cite{deaton1992understanding, friedman2018theory, ch16}, and the neoclassical consumption model \cite{butler2001neoclassical}.

In the PIH model, individuals look into the future and calculate the amount to consume at a particular time by considering the expected average income over time \cite{friedman1957permanent}. The life-cycle model shares similarities with the PIH but introduces the concept that individuals have a defined time frame (as opposed to the indefinite lifespan in PIH) during which they accumulate assets. The goal of this model is to maximize the gain over assets within the given time frame. The neoclassical model is another approach to understanding intertemporal consumption decisions, drawing on principles from neoclassical economics. The income fluctuation problem (IFP) is also a related consumption model with an infinite time horizon optimization with income uncertainty and an upper bound on consumption (the upper bound hinders any consumption more than the amount of assets individuals own, i.e., one cannot consume more than what they possess) \cite{ma2020income, sargent2014quantitative, deaton1989saving, den2010comparison, kuhn2013recursive, rabault2002borrowing, reiter2009solving, schechtman1977some}.


\paragraph{Online Learning in Investment.} In the field of investment, where high volatility and real-time information availability are prevalent, there is a demand for an online model that can allocate investments among a set of assets and maximize cumulative wealth through sequential optimizations. This application of online learning is commonly known as online portfolio selection \cite{dochow2016online, xi2023online, li2014online, li2018online}. It represents an algorithmic trading strategy in online learning, wherein future prices of risky assets are predicted using historical price information. Subsequently, online learning algorithms optimize the portfolio by employing loss functions tailored to specific financial objectives, ultimately aiming to achieve maximum wealth.

While investment models \cite{Bayratkar2012, grandits2015optimal} have valuable components, including modeling of
uncertainty, they can only imperfectly model consumption. Investment models involve strategic decisions about a collection of financial investments (i.e., portfolio allocation) between risky and riskless investments rather than decisions about consumption. An individual can
control their chances of poverty or financial ruin in these models through the reallocation of
resources, which has no analog in managing day-to-day decisions on how much to save or consume. Investment models also do not enforce constraints on consumption.

\paragraph{Reinforcement Learning and Lookahead.}

The concept of lookahead has garnered significant attention in the recent landscape of reinforcement learning (RL). Generally, in RL models, recent works look into the idea of $H$-step lookahead. In $H$-step lookahead~\cite{sikchi2022learning}, the learner has a learned dynamic model and that is used to calculate the action sequence to a horizon of size $H$ to find the optimal policy that maximizes the cumulative result. There are also approaches where the learner incorporates a greedy real-time dynamic programming algorithm, replacing the greedy step with an $H$-step lookahead policy \cite{efroni2020online}. The works of~\cite{moerland2020think, sikchi2022learning, dalal2021improve, efroni2018beyond} study the properties of $H$-step lookahead where the horizon is of a fixed size. The idea of $H$-step lookahead in RL is similar to the lookahead employed in our models. In RL models explored in these works, there is a reliance on well-defined states and actions, involving the learning of model dynamics on the state-action space, utilizing lookahead. This learned information is then used in sampling the next state. In contrast, our models operate with no explicitly defined states.

\section{Determining Consumption and Savings: An Adaptive Online Algorithm}
\label{sec:simulation-background}

A crucial component of a framework designed to investigate the significance of advance notice (future lookahead) involves creating a model that captures consumption under lookahead while accommodating uncertainty. Specifically, it investigates how individuals, facing financial uncertainty, make decisions regarding the amount to consume and save. Although various models attempt to represent consumption under uncertainty, they all fundamentally rely on the concept of discounted utility, which is the most common model in economics for understanding the interplay between consumption and savings. In a discounted utility model, the agent, at each time step, consumes a certain amount $c$ and receives utility $u(c)$ from some concave function $u$. The objective is to devise a policy to determine a consumption value $c$ in a way that maximizes the total discounted utility. After this brief introduction, we will formally articulate the specific cases we study, assumptions, and models in the following section (\S \ref{sec:frameworks}).

\subsection{Our Models}
\label{sec:frameworks}

\newcommand{\dY}{\mathcal{D}_Y}
\newcommand{\dR}{\mathcal{D}_R}

In this section, we introduce the main models that we study in this work. 
We consider both deterministic and stochastic models that use a discount factor $\beta$ to compute the utility. We assume that time is discretized into integer steps, and let $T$ be the effective overall job timeframe (in other words, the time horizon for the algorithm). Let $a_{t},c_t, y_t$ be the assets, consumption, and income at time $t$, respectively. We also let $R_t$ denote the gain from assets (independent of the size of the assets), i.e., the appreciation/depreciation rate of assets. Let $u(.)$ (in our context, we employ $u(c) = \sqrt{c}$ which is in the class of isoelastic utility functions that are commonly used in macroeconomics) be the utility. Let $\dY$ be the income distribution and $\dR$ be the returns distribution.

In all the models, the goal of the algorithm is to maximize the total utility. More formally, it is to solve the following optimization problem:
\begin{align*}
\max ~\sum_{t=1}^T & ~\beta^{(t-1)} u(c_t) \\ 
\text{ subject to: }~  a_{t+1} &= R_t(a_t-c_t)+y_{t} \\ 
0 &\leq c_t \leq a_t
\end{align*}

The constraint $0 \leq c_t \leq a_t$ ensures that the worker 
consumes from the assets available to them. This constraint ensures the worker could consume without going to ruin. 
The model equation $a_{t+1} = R_t(a_t-c_t)+y_{t}$ shows how assets evolve over time given the consumption and income.

\paragraph{The offline or deterministic model.} In the offline model, the income and return values $y_t, R_t$ are known a priori to the algorithm (as, of course, is the starting asset value, $a_1$). This is the most common model considered in traditional economic theory, and the solution can be found using dynamic programming. The ``states'' in the dynamic programming simply correspond to the time step of interest and the total assets remaining. 

\paragraph{The stochastic model.} In the stochastic model, the income and return values are stochastic and they come from known distribution. In this case, an algorithm can optimize the modified objective: 
\[ \max \sum_{t=1}^T ~\beta^{(t-1)}~  \mathbb{E} (u(c_t)) \]
The parameters and the constraints follow the same definitions as in the deterministic case, with the caveat being that $y_t, R_t$ are updated using the ``realized'' values (not their expectations). Note that this is already an online algorithm: at every $t$, the algorithm computes the consumption value using the expected $y, R$ for the future, but as the $y_t, R_t$ values get updated, the algorithm changes its behavior.   

Finally, to explore the effects of future information and lookahead, we consider a combination of the two models. 

\subsection{Online Consumption Algorithm with Lookahead}

We finally consider a model that has ``limited determinism'' controlled by the extent of lookahead, and the rest of the process is stochastic. Formally, at any time $t$, the algorithm knows the exact values of $y_{t+i}$ and $R_{t+i}$, for $i = 1, 2, \dots, \tau$, for some lookahead parameter $\tau$.\footnote{It is also interesting to consider further hybrid models; e.g., the algorithm has a lookahead over $y_t$, but not over $R_t$. This is also realistic in practice since return rates are governed by the market while income lookahead can be controlled by the employer. We do not consider such models in this work.} Additionally, it knows the distribution of the parameters for later time steps. Once again, in contrast to offline economic models, our online adaptive model finds an adaptive consumption policy, i.e., a policy that changes over time given the lookahead as more information arrives. 

The algorithm itself is a straightforward combination of the deterministic and stochastic cases. At each step, the algorithm computes:
\[ \max \sum_{t=0}^\tau ~\beta^t u(c_t) + \mathbb{E}\left(\sum_{t=\tau+1}^T \beta^t u(c_t)\right) \]
where $\mathbb{E}$ is the expectation over the stochasticity in income and returns beyond the lookahead horizon.

Note that the algorithm incorporates new data (as well as lookahead information) as it arrives, which is why it is an online algorithm. 
For simplicity, we assume that a lookahead of $\tau$ implies the agent is aware of the exact return and income values for the next $\tau$ time steps, and for the time steps beyond $\tau$, distributions of these parameters are known. 

\begin{algorithm}
\caption{Online consumption with lookahead}\label{alg:online-cons} 
\begin{algorithmic}[1]
\State Let $\tau$ be the lookahead, $x_0$ be the initial assets
\State Solve the following optimization problem and save the maximum utility for each $x,t$ in a table $V$ where $V[x,t]$ gives the utility of $x,t$ pair,\label{lne:opt-stoch}
\begin{align*}
&\max \mathbb{E}\left(\sum_{t=0}^T \beta^t u(c_t)\right) \\
&x_{t+1} = R_t (x_t-c_t)+y_t \text{ for all } t\\
&y_t \in \dY, R_t \in \dR \text{ for } t \in \{0,\dots,T\}
\end{align*} 
\For{$r=0$ to $T$}
\State Solve the following optimization problem to get $c_r$ given $x_{r}$ using dynamic programming,\label{lne:opt-det}
\begin{align*}
&\max \sum_{t=r}^{r+\tau} \beta^t u(c_t) + V(x_{\tau+1+r},\tau+1+r) \\
&x_{t+1} = R_t (x_t-c_t)+y_t \text{ for all } t\\
&y_t,R_t \text{ exactly known for all } t \in \{r,r+1,\dots,r+\tau\}\\
\end{align*}
\State Set $x_{r+1} = R_r(x_r-c_r) + y_r$
\EndFor
\State \Return The sequence of $c_t$s at each time $t$
\end{algorithmic}
\label{algo:online}
\end{algorithm}

For this model, we define Algorithm \ref{alg:online-cons}. This algorithm solves a stochastic dynamic programming (DP) at the start (in Line~\ref{lne:opt-stoch}) where it populates the table $V$ (of size $* \times T$ where $*$ is a placeholder for the size of the assets, with dynamic resizing as the bounds for the assets increase) with the maximum utility values for each asset and time point pair in the possible paths of the stochastic system. With this in hand, at time point $t$, the goal is to run a deterministic DP in Line~\ref{lne:opt-det} using the precalculated table $V$ and use that information to calculate the optimal consumption choice at that time point $t$. This gives us an online algorithm that yields a new consumption at each time point $t$, based on the set of historical choices as well as the lookahead information available.

\section{Theoretical Analysis}
\label{sec:theory}

\newcommand{\stc}{\ensuremath{{\hat{c}}}}
\newcommand{\wtc}{\ensuremath{{\widehat{c}}}}
\newcommand{\ttc}{\ensuremath{{\widetilde{c}}}}

In this section, we establish a few theoretical results on the consumption behavior and the ``power of lookahead''. We analyze the advantage that a worker privileged with lookahead, could attain compared to an underprivileged worker without lookahead. Let $y_t, x_t, c_t$ be the income, assets and consumption at time $t$. Let $T$ be the timeframe of the job and $u(c) = \sqrt{c}$ be the utility gained from consuming $c$. 

Our first result shows that even in the very simple setting of $\beta=1$ (no discount factor), and income $y_t$ in the range $[0,Y]$, the difference between the total utility of algorithms with a $k$-step lookahead and without lookahead is $\Omega(k \sqrt{Y})$. This is true not only of the algorithms we study, but {\em any} algorithm. This indicates that there are instances where a worker with lookahead privileges gains an edge over an underprivileged worker without lookahead, and the advantage grows linearly with the level of lookahead.

\begin{theorem}\label{thm:main-lookahead}
Consider two individuals, one with a lookahead of $k$ steps and one with no lookahead. Let $c_1,c_2, \dots, c_T$ be the consumption of the individual with lookahead $k$ and $z_1,z_2,\dots,  z_T$ be the consumption of the individual with no lookahead. Then, there exist instances where all incomes are in the range $[0,Y]$, such that
\[\sum_{t=1}^T \sqrt{c_i}-\sum_{t=1}^T \sqrt{z_i} \ge \Omega(k \sqrt{Y})\]
\end{theorem}
While our lower bound result is strong for large values of $k$, it is not very strong for small values. This is partly because we want to emphasize that a gap that grows with $k$ holds even with $R_t = \beta=1$.

\begin{proof}
As discussed above, we consider a very simple setting: $\beta = 1$, returns $R_t=1$ for all $t$. Further we will assume that the individuals start with $a_1=0$ (no initial assets). 

We consider the following input. Let $Y$ be any parameter $>0$.
\[  y_t = \begin{cases} Y ~ \text{ for $t \le k/2$} \\ x\cdot Y \text{ for $k/2 < t \le k$} \end{cases} \]
where $x$ is a value uniformly sampled from $[0,1]$. Note that both the individuals (the one with and without lookahead) know this input distribution. For simplicity, we also assume that the total time horizon $T$ equals $k$ (this assumption can be easily removed by setting $y_t = \frac{(1+x)}{2} Y$ for all $t > k$). As a final simplification, since the incomes can all be scaled, we can assume that $Y=1$ for the remainder of the proof. 

First, consider the individual with $k$ lookahead. Note that they can see the value of $x$, and thus they can consume an amount $\frac{(1+x)}{2}$ at every time step. This is feasible because the first $k/2$ steps have income $y_t = 1$, and so the assets remain above the consumption at all time steps. This yields a total utility (recalling that $T=k$) of $k \sqrt{\frac{1+x}{2}}$. 

Now, consider an individual who does not have any lookahead. Intuitively, they cannot guess the value of $x$, and thus consuming $\frac{1+x}{2}$ is not feasible. But note that the individual may use some complex (possibly randomized) algorithm that consumes non-uniformly and achieve a high total utility. We show that this is not possible.

The starting point of the proof is the classic minmax theorem of Yao~\cite{motwani-raghavan}: for a given input distribution, an optimal algorithm for inputs from drawn this distribution, is a deterministic algorithm. In other words, in order to prove our desired lower bound, it suffices to restrict ourselves to deterministic algorithms and prove a bound on the difference in total utility, in expectation over the choice of $x$. For any deterministic algorithm, since in time steps $1, \dots, (k/2)$, the algorithm only sees income of $1$, the values consumption $z_1, z_2, \dots, z_{(k/2)}$ are all fixed. Let $S = z_1 + z_2 + \dots + z_{(k/2)}$. 

First, suppose it so happens that
\begin{equation}\label{eq:lb-good-condition}
\left| ~ S - \frac{k}{2} \cdot \frac{1+x}{2}  \right| > c \cdot k,
\end{equation}
for some parameter $c$. In this case, we will argue that $\sum_i \sqrt{z_i}$ is significantly smaller than $k \sqrt{\frac{1+x}{2}}$. The starting point is the following inequality about the strict concavity of the square root function:

\begin{lemma}\label{lem:lb-helper}
Let $a \in (1/2, 1)$ be a constant, and let $w \in (0,1)$. Then we have:
\[ \sqrt{w} \le \sqrt{a} + \frac{1}{2\sqrt{a}} (w-a) - \frac{1}{8} (w-a)^2.  \]
\end{lemma}
The proof follows by a simple calculation. 

\begin{proof}
We have:
\begin{align*}
\sqrt{w} - \sqrt{a} - \frac{1}{2\sqrt{a}} (w-a) &= (w-a) \left( \frac{1}{\sqrt{w}+\sqrt{a}} - \frac{1}{2\sqrt{a}} \right) \\
& = \frac{ (w-a)(a-w) }{2\sqrt{a}(\sqrt{w}+\sqrt{a})^2} \\
& = - \frac{ (w-a)^2 }{2\sqrt{a}(\sqrt{w}+\sqrt{a})^2}
\end{align*}
The denominator is $\le 8$, and thus by rearranging, the inequality follows.
\end{proof}

Now, let us write $a = \frac{1+x}{2}$. By assumption, we have that $| z_1 + z_2 + \dots + z_{k/2} - \frac{k}{2} a| > ck$, and thus we have 
\begin{equation}\label{eq:lb-variation} \sum_{i \le k} |z_i - a| > ck.
\end{equation}
Next, we can use Lemma~\ref{lem:lb-helper} to conclude that
\[  \sum_{i \le k} \sqrt{z_i} \le k \sqrt{a} + \frac{1}{2\sqrt{a}} \sum_{i \le k} (z_i -a) - \frac{1}{8} (z_i - a)^2. \]
Now, since the consumption cannot be larger than the overall income (which is $ka$), the middle term on the RHS is $\le 0$. Thus, we have
\[ \sum_{i \le k} \sqrt{z_i} \le k \sqrt{a}  - \frac{1}{8} (z_i - a)^2. \]
Next, by the Cauchy-Schwartz inequality and~\eqref{eq:lb-variation},
 \[ \sum_i (z_i - a)^2 \ge \frac{1}{k} \left( \sum_{i} |z_i -a| \right)^2 > c^2 k. \]
Together, the above inequalities imply that $\sum_i \sqrt{z_i} \le k\sqrt{a} - \frac{c^2 k}{8}$. 
This shows that if the values $z_i$ chosen by the deterministic algorithm satisfy \eqref{eq:lb-good-condition}, then the algorithm with no lookahead has total utility $\Omega(k)$ worse than an algorithm with lookahead. 

The final step is to prove that if $x$ is chosen at random from $(0,1)$, the condition~\eqref{eq:lb-good-condition} holds with a constant probability for some $c>0$. Since $S$ is fixed, the condition is equivalent to $| \frac{2S}{k} - \frac{1+x}{2}| > 2c$. Equivalently, $|\frac{4S}{k} - 1 -x| > 4c$. For any fixed $\alpha$, if $x \sim_{\text{uar}} (0,1)$ the probability that $|\alpha - x| \le 1/3$ is clearly $\le 2/3$. Thus, the condition above must hold with $c = 1/12$, with probability at least $1/3$. 

Putting everything together, we have that with probability $1/3$, the no-lookahead algorithm is $\Omega(k)$ worse than the algorithm with lookahead, and it can never be better. Thus the \emph{expected} difference between the total utilities is also $\Omega(k)$. Yao's minmax theorem implies that this lower bound also holds for any (possibly randomized) algorithm. 
\end{proof}

Theorem~\ref{thm:main-lookahead} shows that an algorithm with lookahead has a provable advantage over an algorithm that knows only the distribution of the incomes, even in the simplest setting where the decay factor $\beta =1$ and the returns are all $1$. Furthermore, the advantage grows \emph{linearly} with the amount of lookahead. In particular, if an individual has infinite lookahead, they can have an advantage of $\Omega(T)$.

Next, we show that when $\beta=1$ and $R_t =1$ for all $t$, the lower bound from Theorem~\ref{thm:main-lookahead}  cannot be improved.

\begin{lemma}\label{lem:k-is-tight}
Suppose $\beta=1$ and the return $R_t=1$ for all $t$. Let $y_1, y_2, \dots, y_T$ be a sequence of income values with $y_t \ge 0$ for all $t$. Then if the consumption sequence of a $k$-lookahead algorithm is $\{c_1, c_2, \dots, c_T\}$, there exists a no-lookahead algorithm whose total utility is $\sqrt{c_1} + \sqrt{c_2} + \dots + \sqrt{c_{T-k}}$.
\end{lemma}
In other words, the difference in the total utility (between the $k$-lookahead and the no-lookahead algorithms) is simply $\sqrt{c_{T-k+1}} + \dots + \sqrt{c_T}$. In settings where all the $c_t$ are of magnitude $O(\text{income})$, this corresponds to $O(k)$ times the square root of the income, which is the lower bound in Theorem~\ref{thm:main-lookahead}.

\begin{proof}
The proof is simple: a no-lookahead algorithm can mimic a $k$-lookahead algorithm, but with a delay of $k$ steps. We will call this the $k$-delay algorithm. Let $c_1,c_2,\dots,c_T$ be the consumption squence of the given $k$-lookahead algorithm. The $k$-delay algorithm is defined as follows,
\begin{itemize}
\item[] For $t = 1$ to $T$:
\begin{enumerate}
\item If $t \le k$, consume $0$
\item Else consume $c_{t-k}$
\end{enumerate}
\end{itemize}

The total utility bound is easy to see. One only needs to check that the algorithm is feasible (i.e., it satisfies the condition that the total consumption until time $t$ is bounded by the total income plus the assets until that time). This is easy to check because the algorithm consumes 0 for the first $k$ steps, while the income $y_t \ge 0$. Since the decay factor $\beta = 1$, delay does not change the utility the algorithm receives.
\end{proof}

\paragraph{Remark.}  We see that the proof relies on having $\beta=1$. If we take into account factors such as inflation (e.g., $\beta = 0.95$), then the ``power of lookahead'' can likely be made much more significant.

\section{Experiments}
In this section, we explore three different sets of experiments regarding the effects of lookahead, the parameters involved in lookahead under uncertain job timetables, and mitigation schemes.

\subsection{Experiment Setup}
\label{sec:expt-setup}

We will first introduce the different elements in our experiment setup.\footnote{Our experimental codebase is available at \url{https://github.com/kanchanarp/COUNTING-HOURS-COUNTING-LOSSES}}


\paragraph{Workers.} Within our framework, agents represent employed individuals who earn weekly income, own assets, and decide whether to consume or save. We create an income distribution of 10,000 individuals using 2019 income data from the US Census Bureau's Annual ASEC survey of the Consumer Price Index (by the IPUMS Consumer Price Survey) \cite{flood2020integrated, dqydj}. 
In the first stage, we eliminate outliers from the income distribution to address challenges related to individuals with exceptionally high or low earnings, which can be challenging to compare due to significant scale differences. We identify outliers by employing the commonly-used inter-quartile range (IQR) proximity rule \cite{dekking2005modern}.
Subsequently, we categorize individuals into four distinct income groups: low incomes ranging from \$1.22 to \$1125.49, low-middle incomes from \$1125.49 to \$2249.75, high-middle incomes from \$2249.75 to \$3374.02, and high incomes from \$3374.02 to \$4498.29. This classification is achieved by partitioning the overall income range into these four segments. To establish asset values, individuals are assigned the median population asset value of \$123,840, derived from median percentile net worth data and median net worth by income percentile data from the Federal Reserve \cite{fedres}.

\paragraph{Minimum Subsistence.} Furthermore, our simulation considers minimum subsistence, i.e., the constraint that the individuals must allocate funds for their minimum basic needs, such as food and shelter \cite{zhang2020fairness, ZIMMERMAN2003233,ALVAREZPELAEZ2005633,Shin2011,shim2014portfolio,ANTONY2019124, dwork2018fairness, miranda2023saving, miranda2020model}. This consideration acknowledges the fundamental requirement for individuals to fulfill their basic necessities as part of the decision-making process regarding consumption and utility maximization. The minimum subsistence values are derived from mean annual expenditures in 2019 from the Consumer Expenditure Surveys of the US Bureau Of Labor Statistics~\cite{minsub}, stratified based on income levels. In essence, individuals with similar income levels are obligated to cover equivalent amounts for their basic needs.

 \paragraph{Shocks.} Shocks, defined as alterations to an agent's financial state due to schedule instability, play a pivotal role in influencing decisions related to consumption and savings. These shocks can either positively or negatively impact income, such as sudden work-hour reductions or increases. The magnitude of income shocks ranges from $-0.4$ to $0.4$, with shocks occurring as $(1+r) \times \text{income}$, where $r$ represents the shock value. The shocks are generated from a Bernoulli process, with the shock size parameter $r$ uniformly sampled from $[-0.4,0.4]$.

\paragraph{Isolating the Impact of Lookahead and Other Parameters.} At each time point (here, a workweek), each individual within an income group shares identical observable features, such as income, shocks stemming from unpredictable schedules, and returns on their assets. The sole point of divergence among individuals within the same group is the extent of lookahead they possess. This deliberate design choice aims to isolate the temporal impact of lookahead on utility, eliminating other financial factors' interference. The overall job timeline spans 26 weeks, equivalent to a 6-month job duration. The return range on saved assets varies between $0.9-1.1$ with an added variance of $\pm 0.05$. The discounting factor $\beta$ is set to the commonly-used value of $0.95$.

\subsection{Future Lookahead: An Empirical Inquiry}
\label{sec:expts-look}

This section aims to examine the utility acquired by individuals under varying levels of lookahead. Specifically, the experiment seeks to compare those with minimal (or no) information about the future, such as protected groups like hourly, part-time workers, and specific racial groups that typically receive limited or no advance notice, against other workers with more foresight. The experiment setup is as explained in \S \ref{sec:expt-setup} and the outcomes of this experiment are depicted in Figure \ref{fig:lookahead_final}.

\begin{figure}[!htbp]
    \centering
  \includegraphics[width=0.45\columnwidth]{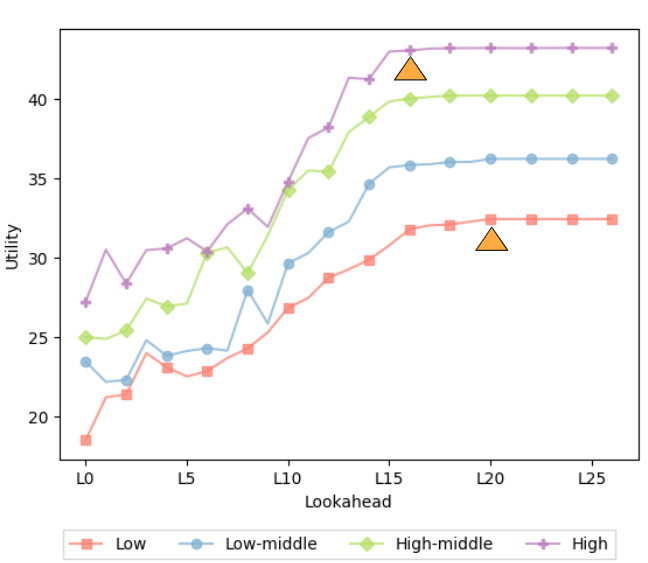}
        \caption{The final utility gained for different levels of lookahead is illustrated for four distinct income classes, each comprising 27 individuals. Workers are of similar features, with variations solely based on the temporal aspect, i.e., the amount of lookahead in their work schedules. The $x$-axis depicts the lookahead value and the $y$-axis represents the total utility at the end of $T$ steps. The two orange triangles are in place to highlight the difference in the level of lookahead needed for approaching the near-maximum utility (near-maximum utility is a utility close to that of a worker with full information at L26) as the income levels change.}
        \label{fig:lookahead_final}
\end{figure}

\paragraph{Analysis.} The insights from Figure \ref{fig:lookahead_final} can be summarized in four key points. 
\begin{itemize}
\item Firstly, workers with lower lookahead and minimal information about future instability experience significantly lower financial utility compared to those with higher lookahead. Unsurprisingly, individuals struggle to manage their finances effectively when confronted with unforeseen schedules. 
\item Secondly, lookahead has a positive impact overall. Workers with more lookahead can efficiently manage their finances, resulting in higher utility. 
\item Thirdly, workers do not require full information about their work schedules in advance. Even beyond the midpoint lookahead (lookahead 13), workers can achieve a utility comparable to those with complete information about their schedules. 
\item Lastly, higher income leads to greater utility, as individuals can consume without concerns. This is evident from the consistent shift of the plots along the $y$-axis. 
\end{itemize}

However, irrespective of income level, having more than the midpoint level of lookahead proves advantageous for individuals. Providing people with advance notice of their schedules can be reasonably implemented for the next 2-3 months of work without requiring employers to furnish complete information at the start of their tenure. Notably, for higher-income workers, less future information is needed to approach the utility values of someone with complete information, as observed by the leftward movement of the orange triangles across all lookahead levels.

\subsection{Dynamics of Asset Appreciation and Depreciation}
\label{sec:expt-return}

Our simulated space provides a comprehensive platform for delving into the intricacies of parameters within work scheduling. This exploration includes understanding the behavior of individuals with diminishing assets compared to those experiencing favorable returns on their savings. 

Therefore, in this section, our objective is to investigate the impact of asset appreciation and depreciation, i.e., positive and negative return rates on workers' assets on decision-making across various levels of lookahead. This also serves as an examination of scenarios wherein the workers are already at a (dis)advantage in terms of assets.

Assets depreciate when their value declines over time, influenced by factors like risky investments, fluctuating market conditions, tax obligations, possessions becoming obsolete, and wear and tear, as seen in vehicles, buildings, and cash saved without earning interest. Conversely, asset appreciation occurs when individuals receive returns on their savings or make profitable investments in stocks, among other factors.

All experimental settings are similar to that of \S \ref{sec:expt-setup} other than the return rates on assets. Here, we compare negative return rates in the range [0.75,0.95] to positive return rates in [1.05,1.25].

\begin{figure*}[!htbp]
    \centering
    \begin{subfigure}[b]{0.49\textwidth}
        \centering
        \includegraphics[width=0.75\columnwidth]{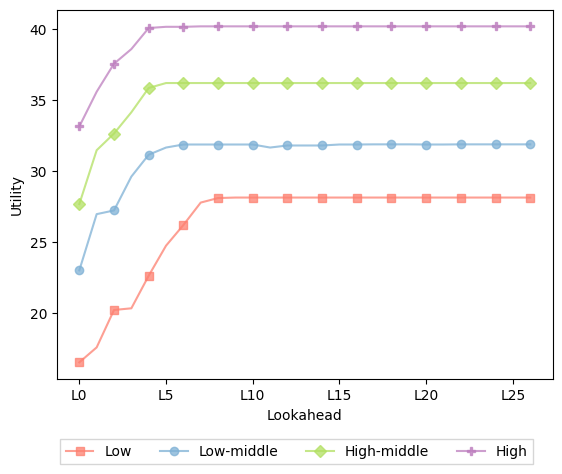}
        \caption{Negative return rate analysis}
        \label{fig:neg-return}
    \end{subfigure}%
    ~ 
    \begin{subfigure}[b]{0.51\textwidth}
        \centering
        \includegraphics[width=0.75\columnwidth]{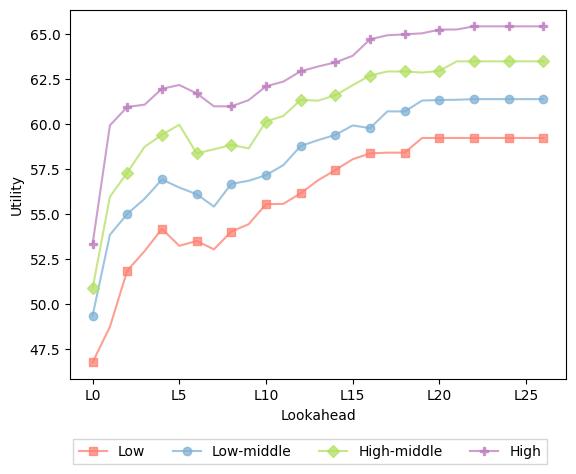}
        \caption{Positive return rate analysis}
        \label{fig:pos-return}
    \end{subfigure}
    
    \caption{Individuals with similar features but varying levels of lookahead under negative return rates ranging from 0.75 to 0.95 on their assets, as well as positive return rates between 1.05 and 1.25 on their assets.} 
    \label{fig:expt-returns}
\end{figure*}

\paragraph{Analysis.} The findings are illustrated in Figure \ref{fig:expt-returns}. Several insights emerge from this analysis.
\begin{itemize}
\item Firstly, workers generally experience higher utility when they encounter favorable return rates. This is evident by comparing the utility range in the negative rates plot, which spans from 15 to 40, to the positive plot, which ranges from 45 to 65. 
\item Secondly, with depreciating assets, individuals benefit significantly from small amounts of lookahead, reaching near-maximum utility. People across income classes achieve a near-maximum utility before Lookahead 10. There is a consistent decline in the lookahead value required to attain near-maximum utility, reaching around lookahead 5 for the highest income class and around 10 for the lowest income class. This observation aligns with the trend observed in the previous section, where more income classes require less future information to reach peak utility values. 
\item Thirdly, in the case of positive returns, individuals have more flexibility in consumption, as they anticipate overall favorable returns ahead. Future information is not as crucial as in the negative returns scenario, where they are not at a disadvantageous situation with depreciating assets. This explains the late convergence of all income classes to the peak utility value (under positive return rates).
\end{itemize}

\subsection{Interventions} 
\label{sec:expts-interventions}

In this section, the goal is to examine the mitigating effects of two intervention scenarios. In terms of intervention policies, simulation provides a valuable sandbox environment for testing that might be challenging or even impossible to explore in the real world. In a simulated work scheduling setting, researchers, policymakers, and practitioners can experiment with various interventions, assess their effects, and fine-tune strategies without real-world consequences. 

\begin{figure}[!htbp]
    \centering
    \begin{subfigure}{.49\textwidth}
    \centering
     \includegraphics[width=0.9\columnwidth]{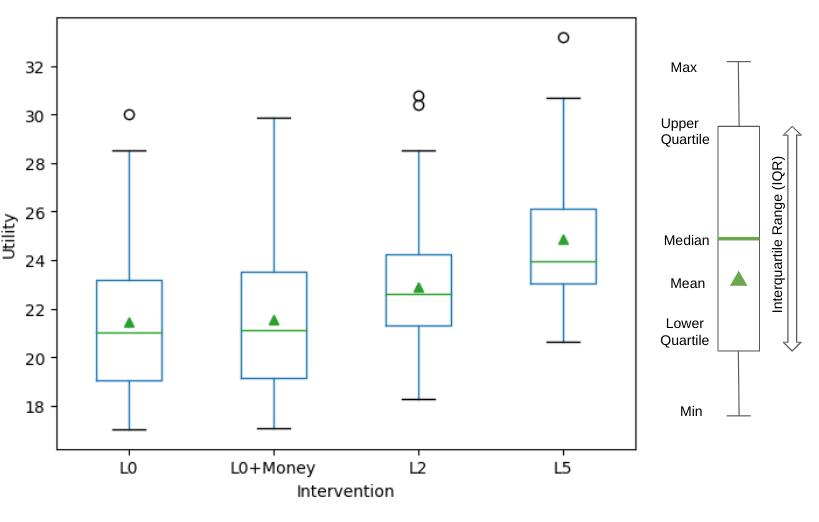}
    \end{subfigure}%
    \begin{subfigure}{.49\textwidth}
     \centering
     \includegraphics[width=0.75\columnwidth]{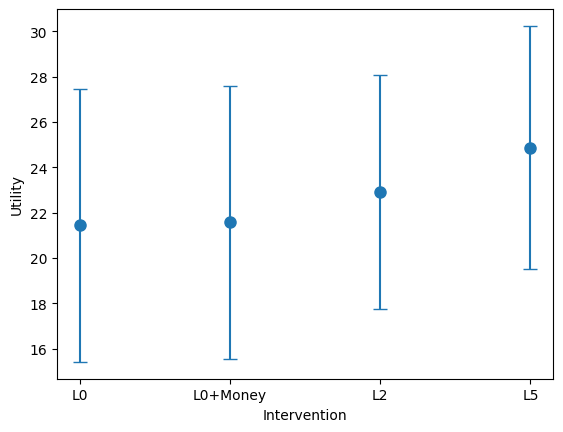}
    \end{subfigure}
        \caption{The box plot and the error bar plot of the statistical distribution of additional gained utilities for various interventions. The interventions considered are L0 + Money, which involves compensation fees for individuals with no future information; L2, assigning at least two weeks of lookahead to all; and L5, assigning at least five weeks of lookahead to everyone. In the box plot, the green triangle represents the mean, and the green line represents the median. In the error bar plot, the blue circle represents the mean and the error bars are 2 standard deviations from the mean (95\% confidence interval).}
        \label{fig:interventions}
\end{figure}

 All experimental settings are similar to \S \ref{sec:expt-setup}. We assign interventions (to a sample of 50 individuals per intervention), as follows: 
\begin{itemize}

    \item \textbf{Compensation:} Workers will be compensated for sudden schedule changes and \emph{on-call} shifts. This policy is modeled after the measures approved by the San Francisco Board of Supervisors, which introduced new protections for retail workers in the city, necessitating employers to provide compensation for unpredictable schedules based on factors such as employment type, hourly rate, hours of work \cite{golden2015irregular}. Inspired by this policy, we disburse twice the amount of earnings back when there is a negative shock.

    \item \textbf{Mandatory minimum advance notice:} Every worker should be entitled to a mandated minimum lookahead, meaning they should be aware of their schedule for the upcoming two weeks. This proposition draws inspiration from the Schedules that Work Act of 2014 (H.R. 5159) presented in Congress, which stipulates that if there are alterations to the schedule and minimum hours, the employer must inform the employee at least two weeks prior to the start point of the new schedule \cite{golden2015irregular}.  Inspired by this policy, we assign workers two weeks and roughly one month (5 weeks) of minimum lookahead, separately.






\end{itemize}

\paragraph{Analysis.} The results for the intervention scenarios are depicted in Figure \ref{fig:interventions} as a box plot to show the distributions (and a separate corresponding error bar plot to capture uncertainty). Three insights can be derived.

\begin{itemize}
\item Firstly, interventions have a positive impact overall, evident from the increased utility across all statistical metrics (mean, median, confidence interval, and all quartiles) post-intervention. 
\item Secondly, while compensation in the form of additional income for unanticipated schedule changes is beneficial, it does not substitute for providing workers with advance notice of their shift schedules. Having knowledge of future plans with the same income but a predictable schedule appears to be more effective for workers' financial well-being than an unforeseen, volatile schedule compensated with fees. A comparison between \emph{L0 + Money} and \emph{L2} indicates that even incorporating two weeks of lookahead is more advantageous than providing compensation for instability without any lookahead. 
\item Thirdly, offering just one month of advance notice results in a notable increase in utility, as seen in the comparison between \emph{L5} and \emph{L0}. Even small amounts of advance notice can significantly enhance utility.
\end{itemize}

\paragraph{Note on the Robustness of Experimental Parameters/Results.} It is important to mention that we conducted these experiments with various random seeds, considering different runs (e.g., with a median asset of approximately \$70,000 \cite{unionwealth}, representing the median working-class wealth, as opposed to the population median), and with other realistic discount and return parameters, as well as varied plausible shock sizes. The overall trends in our results in the previous sections remain consistent as long as the chosen parameters fall within more realistic ranges. If one opts for more extreme and unrealistic parameters, the distinctions observed in \S \ref{sec:expts-look}, \ref{sec:expt-return}, and \ref{sec:expts-interventions} will become even more pronounced.

\section{Discussions and Limitations}
\label{sec:future-work}

The primary contribution of our paper lies in the development of an online framework that delves into the financial insecurity stemming from work schedule instability. This framework addresses two crucial questions pertinent to work scheduling:

\begin{itemize}
    \item How can we formally and empirically assess the negative effects of work schedule instability on income and employment?
    \item How can this in turn help with informing the development of effective mitigation policies and regulations in the workplace?
\end{itemize}

 We provide analytical insights into how lookahead significantly enhances individuals' utility, with a focus on the effectiveness of increased lookahead. Our empirical investigation employs simulations and explores two distinct intervention strategies aimed at mitigating the adverse effects of schedule instability. While our model is being deployed in a work setting, it applies to any scenario that involves temporal uncertainty with the possibility of improvement with lookahead.


\vspace{0.5cm}

Our paper, however, has a number of limitations.

\paragraph{Simulation.} A simulation is only as good as the models used to build it and relies on idealized and formal models of human behavior that are simplifications. A simplified model however is useful for an inquiry into \emph{limits}: in our work, we show that even under ideal models of human utility-maximizing behavior, the lack of predictability and lookahead has concrete consequences for financial stability.

\paragraph{Homogeneity.} Our simulation assumes a level of societal uniformity where all individuals in the same income group face similar external forces and constraints, with temporal differences being the primary distinguishing factor. However, this approach overlooks other forms of societal inequities. Alternatively, if there exists diversity among individuals in managing their finances, the system could incorporate this heterogeneity into its decision-making processes.

Additionally, the impact on one's income (here, due to temporal instability) can vary significantly even for the same job due to various factors such as gender \cite{blau2017gender, reader2022models}, ability \cite{whittaker2019disability}, race \cite{zwerling1992race}, and health status \cite{hicken2014racial, gupta2019equalizing}. Furthermore, the likelihood of benefiting from public policies aimed at addressing these disparities may also vary among different demographic groups.

\paragraph{On Algorithmic Scheduling.}  In the digital age, algorithmic processes and digital technologies are becoming integral components across various workforce sectors. This ranges from business process management, where automation streamlines routine tasks, to the utilization of automated tools for conducting work shift scheduling.

The integration of algorithms and digital tools results in efficiency and productivity across various aspects of the contemporary workforce. For example, the benefits of utilizing labor-scheduling tools instead of basic spreadsheets or manual methods are obvious: it serves as a time and effort-saving solution, offering enhanced control over the scheduling process to both managers and employees. Despite the positive impact of this technology, there are instances where it may introduce new challenges \cite{cons_res}.

Scheduling analytics have the potential to enable operators to exploit their workforce unfairly. Reports from The New York Times point to scheduling software as a factor behind more unpredictable scheduling practices, particularly for low-wage workers in the service industry \cite{cons_res}.  From the standpoint of employees, this unpredictability in scheduling is perceived as a negative attribute. It is an instability enforced by employers, often utilizing workforce management technology and algorithms. Therefore, it would be beneficial to enhance our modeling approach by incorporating more provisions to account for the intricate details and mechanisms of algorithmic scheduling parameters and operations.

\paragraph{Future Directions.} To improve this framework, future enhancements could involve the integration of additional factors influencing financial dynamics, which are correlated with temporal uncertainty. This may include incorporating constraints on retirement savings, distinguishing between risky and riskless assets, and accounting for broader socioeconomic changes, such as alterations in workplace organizational structures or shifts in job locations.

Moreover, our framework presents an avenue for examining fairness and various forms of inequity in automated shift scheduling, particularly concerning sub-populations with different sensitive attributes. It also opens up possibilities for investigating targeted interventions within specific work settings or environments (e.g., specific interventions designed for the food and retail sectors), aiming to alleviate the impacts of biased decision-making on individuals over extended time horizons.

\cleardoublepage
\bibliographystyle{unsrt}  
\bibliography{references}

\end{document}